\documentclass[letterpaper, 10 pt, conference]{ieeeconf}  

\IEEEoverridecommandlockouts                              

\usepackage{graphicx}
\usepackage{amsmath}

\usepackage{amsthm}
\usepackage{multirow}
\usepackage{booktabs}
\usepackage{bm}
\usepackage{comment}
\usepackage{xcolor}

\newtheorem*{proposition}{Proposition}

\usepackage[pagebackref=true,breaklinks=true,letterpaper=true,colorlinks,bookmarks=false]{hyperref}

\title{\LARGE \bf
CoL3D: Collaborative Learning of Single-view Depth and Camera Intrinsics for Metric 3D Shape Recovery}

\author{Chenghao Zhang$^{1}$\thanks{$^{1}$ Alibaba Cloud Computing}, Lubin Fan$^{1*}$, Shen Cao$^{1}$, Bojian Wu$^{2}$\thanks{$^{2}$ Independent Researcher}, and Jieping Ye$^{1}$\thanks{* Corresponding Author}}

\begin{document}

\maketitle
\thispagestyle{empty}
\pagestyle{empty}

\begin{abstract}
Recovering the metric 3D shape from a single image is particularly relevant for robotics and embodied intelligence applications, where accurate spatial understanding is crucial for navigation and interaction with environments. Usually, the mainstream approaches achieve it through monocular depth estimation. However, without camera intrinsics, the 3D metric shape can not be recovered from depth alone. In this study, we theoretically demonstrate that depth serves as a 3D prior constraint for estimating camera intrinsics and uncover the reciprocal relations between these two elements. Motivated by this, we propose a collaborative learning framework for jointly estimating depth and camera intrinsics, named \emph{CoL3D}, to learn metric 3D shapes from single images. Specifically, CoL3D adopts a \textit{unified} network and performs collaborative optimization at three levels: depth, camera intrinsics, and 3D point clouds. For camera intrinsics, we design a canonical incidence field mechanism as a prior that enables the model to learn the residual incident field for enhanced calibration. Additionally, we incorporate a shape similarity measurement loss in the point cloud space, which improves the quality of 3D shapes essential for robotic applications. As a result, when training and testing on \textit{a single dataset} with \textit{in-domain settings}, CoL3D delivers outstanding performance in both depth estimation and camera calibration across several indoor and outdoor benchmark datasets, which leads to remarkable 3D shape quality for the perception capabilities of robots.

\end{abstract}

\section{Introduction}
Recent years have seen significant advancements in understanding 3D scene shapes, particularly in the context of robotics and embodied intelligence~\cite{gothoskar20213dp3,jiang2023self}. For robots to effectively interact with their environments, accurate perception of 3D geometry is essential. Depth sensing serves as a crucial component, providing the distance of each point in the scene from the camera, while camera intrinsics play a vital role in mapping these depths to positions in a 3D space. When combined, these elements enable robots to recover metric 3D scene shapes, fostering enhanced spatial awareness and facilitating various tasks such as navigation, manipulation, and interaction with objects.

Previous works on estimating depth maps or camera intrinsics from a single-view image developed independently along two parallel trajectories. A wave of learning-based methods has promoted the development of the respective tasks, where monocular depth estimation (MDE) primarily focuses on the design of network structures~\cite{adabins,yuan2022neural,li2023depthformer,piccinelli2023idisc,iebins} and single-view camera calibration focuses on the implicit representation of intrinsics~\cite{jin2023perspective,zhu2024tame}. Recent approaches~\cite{Facil_2019_CVPR,Yin_2023_ICCV,Guizilini_2023_ICCV} have incorporated explicit consideration of camera intrinsics into MDE models. They have shown that camera intrinsic enforces MDE models to implicitly understand camera models from the image appearance and then bridges the imaging size to the real-world size. 
Yet, the effectiveness of them depends on unavailable accurate camera intrinsics.

\begin{figure}[t]
\centering
\includegraphics[width=1.0\linewidth]{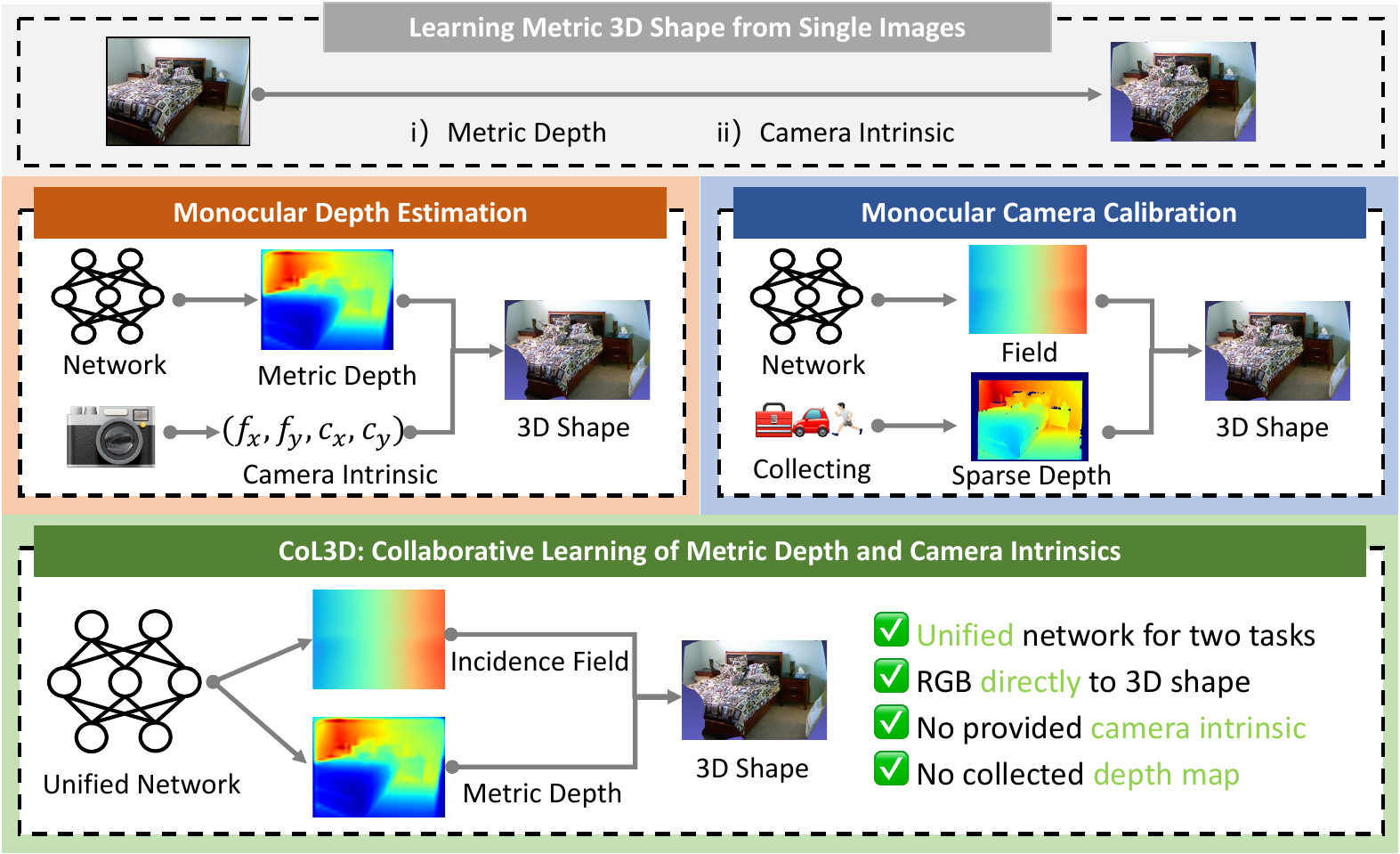}
\vspace{-5mm}
\caption{Comparison of our collaborative learning framework with single-task monocular depth estimation and camera calibration.}
\label{fig:intro_com}
\vspace{-7mm}
\end{figure}

In this study, we explore the reciprocal relations between depth and camera intrinsics from another perspective. We theoretically show that the camera intrinsics can be determined from the depth map given the size of reference objects, which suggests that depth serves as a 3D prior constraint for the estimation of camera intrinsics. These two aspects demonstrate that depth and camera intrinsics are complementary and have a synergistic effect on each other. 

Inspired by this insight, we propose a collaborative learning framework for joint estimation of depth maps and camera intrinsics from a single-view image, named \textbf{CoL3D}. In this framework, the two branches share a unified encoder-decoder network and predict the depth map and the implicit representation of camera intrinsics, \emph{i.e.}, incidence field~\cite{zhu2024tame}, respectively. Fig.~\ref{fig:intro_com} shows the comparison of CoL3D with previous single-task MDE and monocular camera calibration methods. By integrating the two tasks into a unified framework, a metric 3D point cloud can be recovered from a single image without providing additional cues during inference.

Specifically, CoL3D consists of the following two key elements, involving camera calibration and 3D shape recovery.
Firstly, inspired by residual learning, we introduce a canonical incidence field mechanism to promote the model to learn a residual incident field. By setting a prior for the camera intrinsics, we not only reduce the difficulty of intrinsics learning but also render the process from the camera intrinsics to the 3D point cloud completely differentiable. 
Secondly, to alleviate distortions of the recovered 3D point cloud, we further design a shape similarity measurement loss in the point cloud space. By optimizing the scene shape in 3D, we enhance the quality of point clouds derived from predicted depth maps and the incidence field.

Owing to our design, the proposed CoL3D achieves remarkable performance on tasks at various levels. For MDE, our method outperforms state-of-the-art \textit{in-domain} metric depth estimation methods on the popular NYU-Depth-v2~\cite{silberman2012indoor} and KITTI~\cite{geiger2012we} datasets, along with estimating accurate camera intrinsics. In terms of camera calibration, our approach attains comparable performance to the state-of-the-art methods on the Google Street View~\cite{anguelov2010google} and Taskonomy datasets~\cite{zamir2018taskonomy}, while also being capable of predicting reasonable depth maps. Thanks to the outstanding performance on both tasks, our method consistently delivers superior point cloud reconstruction quality on popular datasets.

To summarize, our main contributions are as follows:
\begin{itemize}
\item We reveal the reciprocal relations between depth and camera intrinsics and introduce the CoL3D framework for the collaborative learning of depth maps and camera intrinsics, enabling metric 3D shape recovery from a single-view image within a unified framework.
\item We propose two strategies to empower the model's capabilities at different task levels, including a canonical incidence field for camera calibration and a shape similarity measurement loss for 3D shape recovery.
\item Extensive experiments show that our approach achieves impressive 3D scene shape quality on several benchmark datasets
while estimating accurate depth maps and outstanding camera intrinsics.
\end{itemize}

\section{Related Work}
\textbf{Single-view 3D Recovery.}
Reconstruction of 3D objects from single images has seen notable progress~\cite{barron2014shape,wang2018pixel2mesh,wu2018learning,popov2020corenet}, delivering intricate models for items like vehicles, furniture, and the human form~\cite{saito2019pifu,saito2020pifuhd}. However, the dependence on object-centric 3D learning priors restricts these techniques to full scene reconstruction for robotics applications, such as autonomous navigation and robotic manipulation. Earlier scene reconstruction methods~\cite{saxena2008make3d} segmented scenes into planar segments to approximate 3D architecture. More recently, MDE has been adopted for 3D shape recovery. LeReS~\cite{yin2021learning} incorporates a point cloud module to deduce focal length but necessitates extensive 3D point cloud data for training, particularly challenging for outdoor environments. Meanwhile, GP2~\cite{patakin2022single} introduces a scale-invariant loss to foster depth maps that conserve geometry, but it fails to ascertain focal length. In contrast, our approach focuses on recovering metric 3D scene structure in indoor and outdoor scenarios through a unified framework.

\textbf{Monocular Metric Depth Estimation.}
CNN-based methods predominantly address MDE as a dense regression task~\cite{eigen2014depth,yuan2022neural,liu2023va,piccinelli2023idisc} or a combined regression-classification task through various binning strategies~\cite{adabins,binsformer,localbins,iebins}. The transition to vision transformers has notably enhanced performance~\cite{yang2021transformer,ranftl2021vision,li2023depthformer}. 
Beyond architectural innovation, another line of work~\cite{bhat2023zoedepth,guizilini2023towards,depthanything} focuses on fine-tuning on the metric depth estimation task by using the relative depth estimation pre-trained model as the cornerstone. These methods continue to improve the benchmark results by leveraging massive training data and powerful pre-trained models. In contrast, we reveal the complementary relationship between depth and camera intrinsics. Our approach, demonstrated through in-domain evaluation using a single dataset, allows for better application to customized datasets and scenes.

\textbf{Single Image Camera Calibration.}
Traditionally, camera calibration relied on reference objects like planar grids~\cite{zhang2000flexible} or 1D objects~\cite{zhang2004camera}. Follow-up studies~\cite{schindler2004atlanta,xu2013minimum,wildenauer2012robust,deutscher2002automatic}, operating under the Manhattan World assumption~\cite{coughlan1999manhattan}, have used image line segments~\cite{von2008lsd,akinlar2011edlines} that meet at vanishing points to deduce intrinsic properties. Recent learning-based techniques~\cite{hold2018perceptual,lee2021ctrl,lee2020neural} loosen these constraints by training on panorama images with known horizon and vanishing points to model intrinsic as 1 DoF camera. A notable trend uses the perspective field~\cite{jin2023perspective} or incidence field~\cite{zhu2024tame} to estimate camera intrinsics with 3 DoF or 4 DoF, respectively. In this work, we take a further step and explore the collaborative learning of depth maps and camera intrinsics utilizing the incident field as a bridge.

\textbf{Combination of Depth and Intrinsics.}
Recent studies~\cite{Facil_2019_CVPR,Yin_2023_ICCV,Guizilini_2023_ICCV} have revisited depth estimation by explicitly incorporating camera intrinsics, particularly focal length, as additional input to learn metric depth. However, focal length is often inaccessible during deployment. The challenge lies in how to jointly learn depth and intrinsics for the accurate recovery of metric 3D shapes. Note that, UniDepth~\cite{piccinelli2024unidepth} addresses this by leveraging considerable and diverse datasets and large-scale backbones. In contrast, in our \textit{in-domain} training and testing settings, we explore the reciprocal relations between depth and camera intrinsics and also achieve impressive performance on \textit{a single dataset}, which offers flexibility to meet various customized requirements.

\begin{figure*}[t]
\begin{center}
\includegraphics[width=1.0\textwidth]{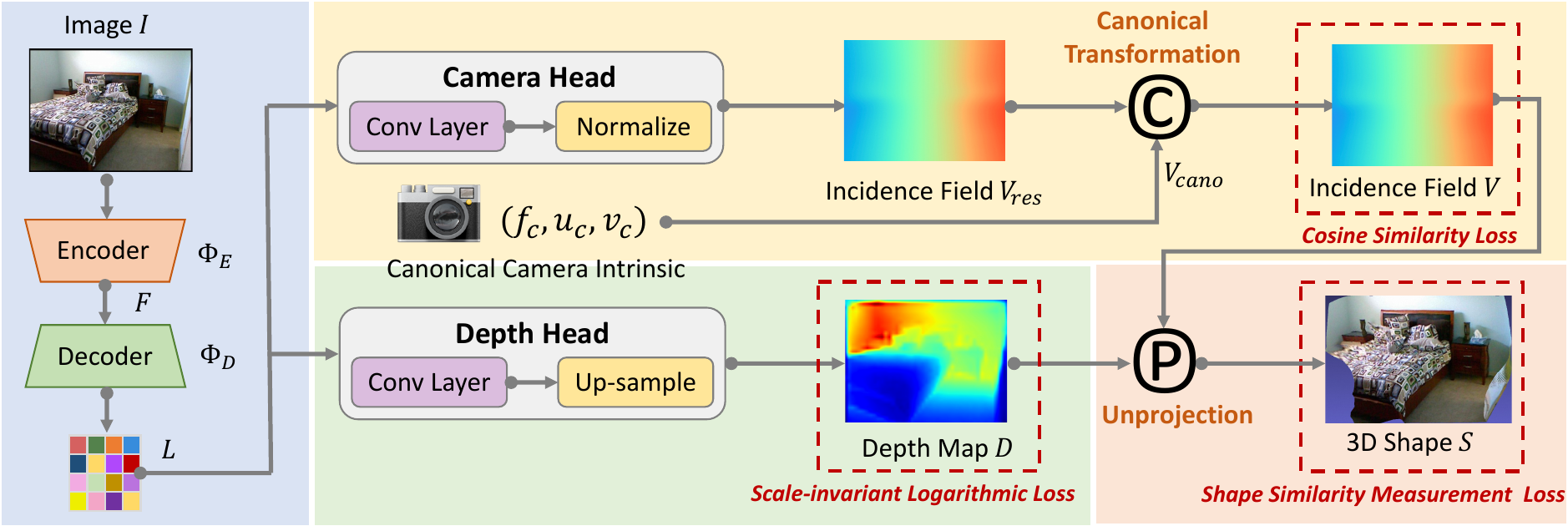}
\end{center}
\vspace{-5mm}
\caption{Overview of the proposed CoL3D framework. It consists of an Encoder and Decoder for latent feature extraction, a Depth Head for depth prediction, and a Camera Head for camera intrinsics estimation. Collaborative learning is performed on the depth map, the incident field, and the 3D point cloud. Note that camera intrinsics are only used for training and are predicted by the model itself at inference.}
\label{fig:overview}
\vspace{-3mm}
\end{figure*}

\section{Preliminary}
\textbf{Problem Statement.}
In this study, we focus on collaborative learning of monocular depth and camera intrinsics to recover a metric 3D shape. We assume a standard camera model for the 3D point
cloud reconstruction, which means that the unprojection from 2D coordinates and depth to 3D points is:
\begin{equation}
\label{eq:main}
x = \frac{u-c_x}{f_x} d, y = \frac{v-c_y}{f_y} d, z = d,
\end{equation}
where $f_x$ and $f_y$ are the pixel-represented focal length along the $x$ and $y$ axes, $(c_x, c_y)$ is the principle center, and $d$ is the depth. The focal length affects the point cloud shape as it scales $x$ and $y$ coordinates. Similarly, a shift of $d$ will result in shape distortions. Previous works~\cite{Yin_2023_ICCV,Guizilini_2023_ICCV} have shown the guiding role of camera intrinsics on depth estimation, and we demonstrate that depth serves as a 3D prior constraint on camera intrinsics estimation through the following proposition.
\begin{proposition}
Given the depth map of an image, the 4 DoF camera intrinsics can be determined by 4 non-overlapping groups of pixels in the image with their Euclidean distances in the 3D space.
\end{proposition}

We provide additional proof in the video attachment. Note that the pixels in the image and their spatial distance generally represent the size and scale of reference objects in the 3D world, like beds or cars.

\textbf{Incidence Field.}
The incidence field~\cite{zhu2024tame} is defined as the incidence rays between points in 3D space and pixels in the 2D imaging plane, which is regarded as a pixel-wise parameterization of camera intrinsics. An incidence ray from a pixel $\mathbf{p^T} = [u\quad v\quad 1]$ in the 2D image space is defined as:
\begin{equation}
    \mathbf{v^T} = [(u-c_x)/f_x\quad (v-c_y)/f_y\quad 1].
\end{equation}
The incidence field $\mathbf{V}$ is determined by the collection of incidence rays associated with each pixel, where $\mathbf{v} = \mathbf{V}(\mathbf{p})$.

\section{Methodology}
Fig.~\ref{fig:overview} shows the overall framework of the proposed CoL3D framework. In the spirit of fully exploring the reciprocal relationship between depth and camera intrinsics, CoL3D achieves knowledge complementarity by sharing the encoder and decoder and employing respective prediction heads. To obtain a better quality of 3D scene shape, we propose the canonical incident field mechanism and the shape similarity measurement loss. The whole framework is optimized at three levels, which are depth, camera, and point cloud. The details are introduced in subsequent sections.

\subsection{Canonical Incidence Field}
\label{sec:4.1}
The elements that compose camera intrinsics usually have specific numerical ranges. For instance, the field of view (FoV) of a standard camera is generally between $40^\circ$ to $120^\circ$, and the optical center is generally near the center of the image. Compared with direct prediction without reference values, setting canonical intrinsic elements as initial values can serve as a prior for incident field learning. Inspired by residual learning~\cite{he2016deep}, we propose to enable the model to learn residuals based on canonical camera intrinsics to reduce the difficulty of incident field learning and thereby improve the performance of camera intrinsics estimation.

We denote the incident field composed of the canonical camera intrinsics elements as \emph{Canonical Incident Field} $\mathbf{V}_{cano}$, which is defined as follows:
\begin{equation}
\mathbf{K}_{cano} = \left[
\begin{array}{ccc}
f_c & 0 & u_c \\
0 & f_c & v_c \\
0 & 0 & 1 \\
\end{array}\right],
\mathbf{V}_{cano}(\mathbf{p}) = \left[
\begin{array}{ccc}
(u - u_c) / f_c \\
(v - v_c) / f_c \\
1 \\
\end{array}\right],
\end{equation}
where $f_c$ represents the canonical focal length along the horizontal and vertical image axes, and $u_c=w/2$ and $v_c=h/2$ represent the coordinates of the canonical principal point. To this end, the Camera Head targets to learn the residual incident field $\mathbf{V}_{res}$ of the ground truth incident field $\mathbf{V}_{gt}$ relative to the canonical incident field $\mathbf{V}_{cano}$. That is to say, $\mathbf{V}_{res} \cdot \mathbf{V}_{cano} = \mathbf{V}_{gt}$.

Using the incident field as an implicit representation of the focal length, the 3D point cloud can be directly obtained from the combination of the incident field with the depth, as illustrated in Eq.~\eqref{eq:main}.
In this way, we achieve full differentiability from the focal length to the 3D point cloud.

\subsection{Shape Similarity Measurement}
\label{sec:4.2}
Typically, evaluation metrics for MDE usually measure the per-pixel estimation error, but cannot evaluate the overall quality of the 3D scene shape. Minor errors within the depth maps may be amplified when converted into 3D space, which may subsequently lead to scene shape distortion. It is a critical problem for downstream tasks such as 3D view synthesis and 3D photography. Potential reasons include depth discontinuities, uneven error distribution, and inaccurate camera intrinsics. 

To improve the quality of the recovered 3D shape, we propose a 3D shape similarity measurement mechanism, aiming to collaboratively optimize the depth map and camera intrinsics in the point cloud space. Specifically, we employ the Chamfer Distance~\cite{borgefors1988hierarchical} as the point cloud similarity metric to calculate the distance between predicted and ground truth 3D point clouds as follows:
\begin{equation}
\mathcal{M}(\mathcal{P}, \mathcal{Q}) = \frac{1}{|\mathcal{P}|} \sum_{p \in \mathcal{P}} \min_{q \in \mathcal{Q}} |p - q|^2 + \frac{1}{|\mathcal{Q}|} \sum_{q \in \mathcal{Q}} \min_{p \in \mathcal{P}} |q - p|^2,
\end{equation}
where $\mathcal{P}$ and $\mathcal{Q}$ represent the sets of points in the predicted and ground truth point clouds, respectively, and $|p - q|$ denotes the Euclidean distance between points $p$ and $q$. This metric effectively measures the average closest point distance between the two point clouds, which has fully differentiable properties for comprehensive 3D shape optimization. 

\subsection{Collaborative Learning Protocol}
\label{sec:4.3}
\textbf{Architecture.}
The proposed CoL3D framework consists of an Encoder Backbone $\Phi_E$, a Decoder Module $\Phi_D$, a Depth Head $\phi_d$, and a Camera Head $\phi_c$ (see Fig.~\ref{fig:overview}). Given an RGB image $\mathbf{I} \in \mathcal{R}^{h \times w \times 3}$ with $w$ and $h$ representing the width and height of the image, we adopt the Swin-Transformer~\cite{liu2021swin} as the encoder, producing features at different scales, \emph{i.e.}, $\mathbf{F} \in \mathcal{R}^{h \times w \times C \times B}$, where $B=4$. The latent feature tensor is obtained as the average of the features $\mathbf{F}$ along the $B$ dimension. The decoder is inspired from iDisc~\cite{piccinelli2023idisc} and is fed with the latent feature, yielding the decoded features $\mathbf{L} \in \mathcal{R}^{h \times w \times C}$. Furthermore, the Depth Head and Camera Head take the decoded features $\mathbf{L}$ as input and estimate the depth map $\mathbf{D} \in \mathcal{R}^{h \times w}$ and incident field $\mathbf{V} \in \mathcal{R}^{h \times w \times 3}$, respectively. The Depth Head consists of a convolutional layer followed by an upsampling layer while the Camera Head changes the Depth Head to output a three-dimensional normalized incident field. The metric 3D shape $\mathbf{S} \in \mathcal{R}^{h \times w \times 3}$ is recovered by the unprojection from the predicted depth map and incidence field.

\textbf{Optimization.}
Collaborative learning is performed at the depth level, camera level, and point cloud level. Following~\cite{yuan2022neural,piccinelli2023idisc,iebins}, we leverage the scale-invariant logarithmic loss for depth estimation,
\begin{equation}
\mathcal{L}_{silog} = \frac{1}{n} \sum_i (\Delta D_i)^2 - \frac{\lambda}{n^2}(\sum_i \Delta D_i)^2,
\end{equation}
where $\Delta D_i = \log \mathbf{D}_i - \log \mathbf{D^*}_i$. Here, $\mathbf{D}$ is the predicted depth, $\mathbf{D^*}$ is the ground truth depth, both with $n$ pixels indexed by $i$, and $\lambda \in [0,1]$. For incidence field learning, we adopt a cosine similarity loss defined as:
\begin{equation}
\mathcal{L}_{cos} = \frac{1}{n} \sum_i (\mathbf{V}_i \cdot \mathbf{V}_{cano})^T \mathbf{V}_i^*,
\end{equation}
where $\mathbf{V}$ is the predicted incidence field, $\mathbf{V}^*$ is the ground truth incidence field. For metric 3D shape learning, define $\mathbf{S}$ the predicted point cloud with predicted depth $d = \mathbf{D}(u,v)$ and estimated camera intrinsic elements $(\hat{c_x}, \hat{c_y}, \hat{f_x}, \hat{f_y})$ and $\mathbf{S^*}$ the ground truth point cloud with ground truth depth $d^* = \mathbf{D^*}(u,v)$ and ground truth camera intrinsic elements $(c_x^*, c_y^*, f_x^*, f_y^*)$ as:
\begin{equation}
\mathbf{S} := \left\{
\begin{array}{lll}
S_x = \frac{u-\hat{c_x}}{\hat{f_x}} d  \\
S_y = \frac{v-\hat{c_y}}{\hat{f_y}} d  \\
S_z = d \\
\end{array}
\right., \mathbf{S^*} := \left\{
\begin{array}{lll}
S^*_x = \frac{u-c_x^*}{f_x^*} d^*  \\
S^*_y = \frac{v-c_y^*}{f_y^*} d^*  \\
S^*_z = d^* \\
\end{array}
\right..
\end{equation}
We utilize the proposed shape similarity measurement as the loss in 3D space:
\begin{equation}
\mathcal{L}_{cd} = \mathcal{M}(\mathbf{S}, \mathbf{S^*}). 
\end{equation}
The overall loss function is formally defined as follows:
\begin{equation}
\mathcal{L} = \alpha\mathcal{L}_{silog} + \beta\mathcal{L}_{cos} + \gamma\mathcal{L}_{cd},
\end{equation}
where $\alpha$, $\beta$, and $\gamma$ are weight parameters.

\section{Experiments}

\subsection{Experimental Setup}

\textbf{Datasets.}
For MDE, we use three benchmark datasets to evaluate our approach, including NYU-Depth V2 (NYU)~\cite{silberman2012indoor}, KITTI~\cite{geiger2012we}, and SUN RGB-D~\cite{song2015sun} datasets. The NYU dataset is divided into 24,231 samples for training and 654 for testing according to the split by~\cite{lee2019big}. The KITTI dataset follows Eigen-split~\cite{eigen2014depth} with 23,158 training images and 652 testing images. The SUN RGB-D dataset is used for zero-shot generalization study and the official 5,050 test images are adopted. For monocular camera calibration, we adopt the Google Street View (GSV) dataset~\cite{anguelov2010google} for evaluation, which provides 13,214 images for training and 1,333 images for testing. We also utilize Taskonomy~\cite{zamir2018taskonomy} dataset for monocular depth z-buffer prediction and single-view camera calibration tasks. The standard \emph{Tiny} splits are adopted with 24 training buildings (250K images) and 5 validation buildings (52K images).

\textbf{Evaluation Metrics.}
For 3D shape recovery quality, we adopt $F1$ score, Chamfer Distance, and the Locally Scale Invariant RMSE (LSIV) metric in~\cite{chen2020oasis}. For MDE, following previous works~\cite{yuan2022neural,piccinelli2023idisc}, the accuracy under threshold ($\delta_i < 1.25^i, i=1,2,3$), absolute relative error (A.Rel), relative squared error (Sq.Rel), root mean squared error (RMSE), root mean squared logarithmic error (RMSE log), and $\log_{10}$ error ($\log_{10}$) metrics are employed. For camera calibration, we convert the focal length to FoV, calculate the angular error, and report two metrics: the mean error and median error following~\cite{zhu2024tame}.

\textbf{Implementation Details.}

CoL3D is implemented in PyTorch. For architecture, we adopt Swin-Transformer as the Encoder and utilize the Internal Discretization in iDisc as the Decoder. The Depth Head and Camera Head mainly consist of convolutional layers, followed by upsampling and normalization, respectively. For training, we use the AdamW optimizer ($\beta_1=0.9, \beta_2=0.999$) with an initial learning rate of 2e-4, and weight decay set to 0.02. As a scheduler, we exploit Cosine Annealing starting from 30\% of the training, with a final learning rate of 2e-5. We run 45k optimization iterations with a batch size of 16 for all datasets. All backbones are initialized with weights from ImageNet-pretrained models. The required training time amounts to 5 days on 8 V100 GPUs. We set $\lambda=0.5$ and the loss weights $\alpha=1$, $\beta=10$, and $\gamma=1$, respectively.

\begin{table}[htb!]
\caption{\textbf{Comparisons of depth estimation on the NYU dataset}.} 
\label{tab:nyu}
\begin{center}
\footnotesize
\vspace{-0.5cm}
\setlength{\tabcolsep}{1mm}{
\begin{tabular}{l | c c c| c c c}	
\toprule
Method & A.Rel $\downarrow$ & RMSE $\downarrow$ & ${\textbf{\rm{log}}_{\bm{{10}}}}$ $\downarrow$ &  $\delta_1$ $\uparrow$ &  $\delta_2$ $\uparrow$& $\delta_3$ $\uparrow$ \\
\midrule
AdaBins~\cite{adabins} &0.103&0.364&0.044&0.903&0.984&0.997\\
P3Depth~\cite{patil2022p3depth}& 0.104&0.356&0.043&0.898&0.981&0.996\\ 
LocalBins~\cite{localbins}&0.099&0.357&0.042&0.907&0.987&0.998\\ 	
NeWCRFs~\cite{yuan2022neural}&0.095&0.334&0.041&0.922&0.992&0.998\\
BinsFormer~\cite{binsformer}&0.094&0.330&0.040&0.925&0.989&0.997\\ 
IEBins~\cite{iebins}& 0.087 &0.314 &0.038 &0.936 &0.992 & 0.998\\
iDisc~\cite{piccinelli2023idisc} & 0.086 & 0.313 & 0.037 & 0.940 & \textbf{0.993} & \textbf{0.999} \\
Metric3D~\cite{Yin_2023_ICCV} & \textbf{0.083} & 0.310 & \textbf{0.035} & \textbf{0.944} & 0.986 & 0.995\\
Unidepth~\cite{piccinelli2024unidepth} & \textcolor{gray}{0.626} & \textcolor{gray}{0.232} & - & \textcolor{gray}{0.972} & - & -\\
\midrule
\textbf{Ours}  & \textbf{0.083}&\textbf{0.294}&\textbf{0.035}&\textbf{0.944}&0.992&\textbf{0.999}\\
\bottomrule
\end{tabular}}
\end{center}
\vspace{-5mm}
\end{table}

\begin{table}[htb!]
\caption{\textbf{Zero-shot generalization to the SUN RGB-D dataset with models trained on NYU.} The maximum depth is capped at 10m.}
\label{tab:nyu_to_sunrgbd}
\begin{center}
\footnotesize
\vspace{-5mm}
\setlength{\tabcolsep}{1mm}{
\begin{tabular}{l |  c c c| c c c}	
\toprule
Method &  A.Rel $\downarrow$ & RMSE $\downarrow$ & ${\textbf{\rm{log}}_{\bm{{10}}}}$ $\downarrow$ &  $\delta_1$ $\uparrow$ &  $\delta_2$ $\uparrow$& $\delta_3$ $\uparrow$ \\
\midrule
AdaBins~\cite{adabins}&0.159&0.476&0.068&0.771&0.944&0.983\\
LocalBins~\cite{localbins} &0.156&0.470&0.067&0.777&0.949&0.985\\
NeWCRFs~\cite{yuan2022neural}& 0.150 & 0.429 & 0.063 & 0.799 & 0.952 & 0.987\\
BinsFormer~\cite{binsformer}&0.143&0.421&0.061&0.805&0.963&0.990\\ 
IEBins~\cite{iebins}& 0.135&0.405&0.059&0.822&0.971&0.993\\
iDisc~\cite{piccinelli2023idisc} & 0.128 & 0.387 & 0.056 & 0.836 & 0.974 & 0.994\\
\midrule
\textbf{Ours} &\textbf{0.127}&\textbf{0.369}&\textbf{0.055}&\textbf{0.849}&\textbf{0.977}&\textbf{0.995}\\
\bottomrule
\end{tabular}}
\end{center}
\vspace{-5mm}
\end{table}

\begin{table}[htb!]
\caption{\textbf{Comparisons of depth estimation on the Eigen split of KITTI dataset}. The maximum depth is capped at 80m.}
\label{tab:kitti}
\begin{center}
\scriptsize
\vspace{-5mm}
\setlength{\tabcolsep}{0.7mm}{
\begin{tabular}{l |  c c  c c | c c c }	
\toprule
Method & A.Rel $\downarrow$ & Sq.Rel $\downarrow$ & RMSE $\downarrow$ & RMSE$_{\log}$ $\downarrow$ & $\delta_1$ $\uparrow$ & $\delta_2$ $\uparrow$& $\delta_3$ $\uparrow$ \\
\midrule
AdaBins~\cite{adabins}&0.058&0.190&2.360&0.088&0.964&0.995&\textbf{0.999}
\\ 
P3Depth~\cite{patil2022p3depth}&0.071&0.270&2.842&0.103&0.953&0.993&0.998
\\
NeWCRFs~\cite{yuan2022neural}&0.052&0.155&2.129&0.079&0.974&0.997&\textbf{0.999}
\\
BinsFormer~\cite{binsformer}&0.052&0.151&2.098&0.079&0.974&0.997&\textbf{0.999}
\\
Metric3D~\cite{Yin_2023_ICCV} & 0.053 & 0.174 & 2.243 & 0.087 & 0.968 & 0.996 & \textbf{0.999} \\
iDisc~\cite{piccinelli2023idisc} & \textbf{0.050} & 0.145 & 2.067 & 0.077 & 0.977 & 0.997 & \textbf{0.999} \\
IEBins~\cite{iebins} & \textbf{0.050} & 0.142 & 2.011 &0.075 &\textbf{0.978} &\textbf{0.998} &\textbf{0.999}\\
Unidepth~\cite{piccinelli2024unidepth} & \textcolor{gray}{0.469} & - & \textcolor{gray}{2.000} & \textcolor{gray}{0.072} & \textcolor{gray}{0.979} & - & - \\
\midrule
\textbf{Ours} & \textbf{0.050}&\textbf{0.140}&\textbf{2.002}&\textbf{0.073}&\textbf{0.978}&\textbf{0.998}&\textbf{0.999} \\
\bottomrule
\end{tabular}}
\end{center}
\vspace{-5mm}
\end{table}

\textbf{Comparison Protocols.} To ensure a fair comparison, we select the state-of-the-art methods that use similar in-domain settings, meaning their training and testing are all conducted on a single dataset. It is worth mentioning that many current models are exploring training on larger datasets with more complex architectures. While we acknowledge that they may perform better in certain cases, their training schemes differ significantly from ours. Our focus is how depth and camera intrinsics can complement each other within in-domain settings, which offer flexibility for customized requirements.

\subsection{Depth Estimation}

Table~\ref{tab:nyu} compares our CoL3D method with in-domain metric depth estimation methods on NYU. CoL3D improves by over 6\% on RMSE and 3\% on A.Rel compared to previous methods. Our method also shows versatility with remarkable depth estimation performance and a mean FoV($^\circ$) error of 0.71. However, there is still a gap compared to depth estimation foundation models like Unidepth~\cite{piccinelli2024unidepth}, which use large-scale datasets.
Tab.~\ref{tab:nyu_to_sunrgbd} presents zero-shot generalization comparisons on SUN RGB-D with models trained on NYU. We achieve the best generalization performance compared to other methods, which suggests that the proposed framework captures better geometric structures in indoor scenes.

\begin{table}[htb!]
\caption{\textbf{Effectiveness of key components on Taskonomy-Tiny}.}
\label{tab:taskonomy}
\begin{center}
\small
\vspace{-5mm}
\setlength{\tabcolsep}{1.8mm}{
\begin{tabular}{c | c c | c | c }	
\toprule
Method  & RMSE $\downarrow$ & $\delta_1  $ $\uparrow$ &  FoV $\downarrow$ & LSIV $\downarrow$ \\
\midrule
MDE w/o Camera Head  & 0.411 & 0.913 & - & - \\
Camera Calibration & - & - & 1.456 & - \\
\midrule
Baseline  & 0.398 & 0.916 & 1.432 & 0.237 \\
Baseline+$\mathbf{V}_{cano}$  & 0.396 & \textbf{0.917} 
 & 1.369 & 0.235 \\
Baseline+$\mathbf{V}_{cano}$+$\mathcal{L}_{cd}$  & \textbf{0.394} & \textbf{0.917} & \textbf{1.342} & \textbf{0.232}\\
\bottomrule
\end{tabular}}
\end{center}
\vspace{-5mm}
\end{table}

\begin{table}[htb!]
\caption{\textbf{Comparisons for monocular camera calibration on GSV}.}
\label{tab:gsv}
\begin{center}
\small
\vspace{-5mm}
\setlength{\tabcolsep}{4mm}{
\begin{tabular}{l | c | c }	
\toprule
Method & Mean $\downarrow$ & Median $\downarrow$ \\
\midrule
Upright~\cite{lee2013automatic} & 9.47 & 4.42\\
Perceptual~\cite{hold2018perceptual} & 4.37 & 3.58\\
CTRL-C~\cite{lee2021ctrl} & 3.59 & 2.72 \\
Perspective~\cite{jin2023perspective} & 3.07 & 2.33\\
\midrule
Ours w/o Asm. & 2.60 & 2.07\\
Ours w Asm. & \underline{2.58} & \underline{2.03}\\
\midrule
Incidence~\cite{zhu2024tame} & \textbf{2.49} & \textbf{1.96} \\
\bottomrule
\end{tabular}}
\end{center}
\vspace{-5mm}
\end{table}

\begin{table}[htb!]
\caption{\textbf{Comparisons of 3D shape quality on the NYU dataset}.}
\scriptsize
\label{tab:nyu_3d_metric}
\begin{center}
\vspace{-5mm}
\setlength{\tabcolsep}{0.5mm}{
\begin{tabular}{l | c c c c c  | c }	
\toprule
Method & $\mathbf{F1}_{0.05}$ $\uparrow$ & $\mathbf{F1}_{0.1}$ $\uparrow$  & $\mathbf{F1}_{0.3}$ $\uparrow$ & $\mathbf{F1}_{0.5}$ $\uparrow$ & $\mathbf{F1}_{0.75}$ $\uparrow$ & $\mathbf{D}_{Cham}$ $\downarrow$ \\
\midrule
BTS~\cite{lee2019big} & 24.5 & 47.0 &  84.4 & 93.6 & 97.2 & 0.169 \\
AdaBins~\cite{adabins} & 24.0 & 47.0 &  84.7 & 94.0 & 97.4 & 0.163 \\ 		
NeWCRFs~\cite{yuan2022neural} &25.5 & 48.6 &  85.4 & 94.4 & 97.6 & 0.156 \\
iDisc~\cite{piccinelli2023idisc} &27.8 & 52.0 &  87.8 & 95.5 & 98.1 & 0.131 \\
IEBins~\cite{iebins} & 28.0 & 52.2 &  88.1 & 95.6 & 98.3 & 0.128 \\
\hline
\textbf{Ours} & \textbf{28.5} & \textbf{52.9}   & \textbf{88.3} & \textbf{96.1} & \textbf{98.7} & \textbf{0.120}\\
\bottomrule
\end{tabular}}
\end{center}
\vspace{-5mm}
\end{table}

The comparison results on the KITTI dataset shown in Tab.~\ref{tab:kitti} further verify the scalability and advantages of our method in outdoor scenes, pushing already low RMSE to a lower level while realizing a mean FoV($^\circ$) error of 1.42 for camera calibration. We claim that the merit of our method lies in its ability to additionally estimate useful camera intrinsics while predicting accurate depths. 
We provide depth visualization comparisons in the video attachment.

\subsection{Camera Calibration}
To evaluate the accuracy of our recovered camera intrinsics, we perform experiments on Taskonomy-Tiny~\cite{zamir2018taskonomy}, which provides ground-truth depth and diverse camera intrinsics satisfying the data requirements. We parse the intrinsics
from the provided camera location, camera pose, an FoV. Tab.~\ref{tab:taskonomy} shows the performance comparison between our collaborative learning framework and each individual task. Our method significantly improves the camera calibration performance compared to performing calibration alone.

Furthermore, we compare the focal length estimation performance on the popular Google Street View benchmark following~\cite{lee2021ctrl}. Note that we employ the off-the-shelf MDE model~\cite{Yin_2023_ICCV} with accurate camera intrinsics involved in GSV to predict depth maps as depth pseudo-labels for collaborative learning since GSV does not provide depth labels. The results in Table~\ref{tab:gsv} demonstrate that our unified framework outperforms most state-of-the-art single-task camera calibration methods. Notably, even when trained with noisy depth pseudo-labels, our approach retains the performance of the Incidence Field method~\cite{zhu2024tame} on camera calibration, while additionally delivering valuable estimated depth maps.

\begin{figure}[htb!]
\begin{center}
\includegraphics[width=0.98\linewidth]{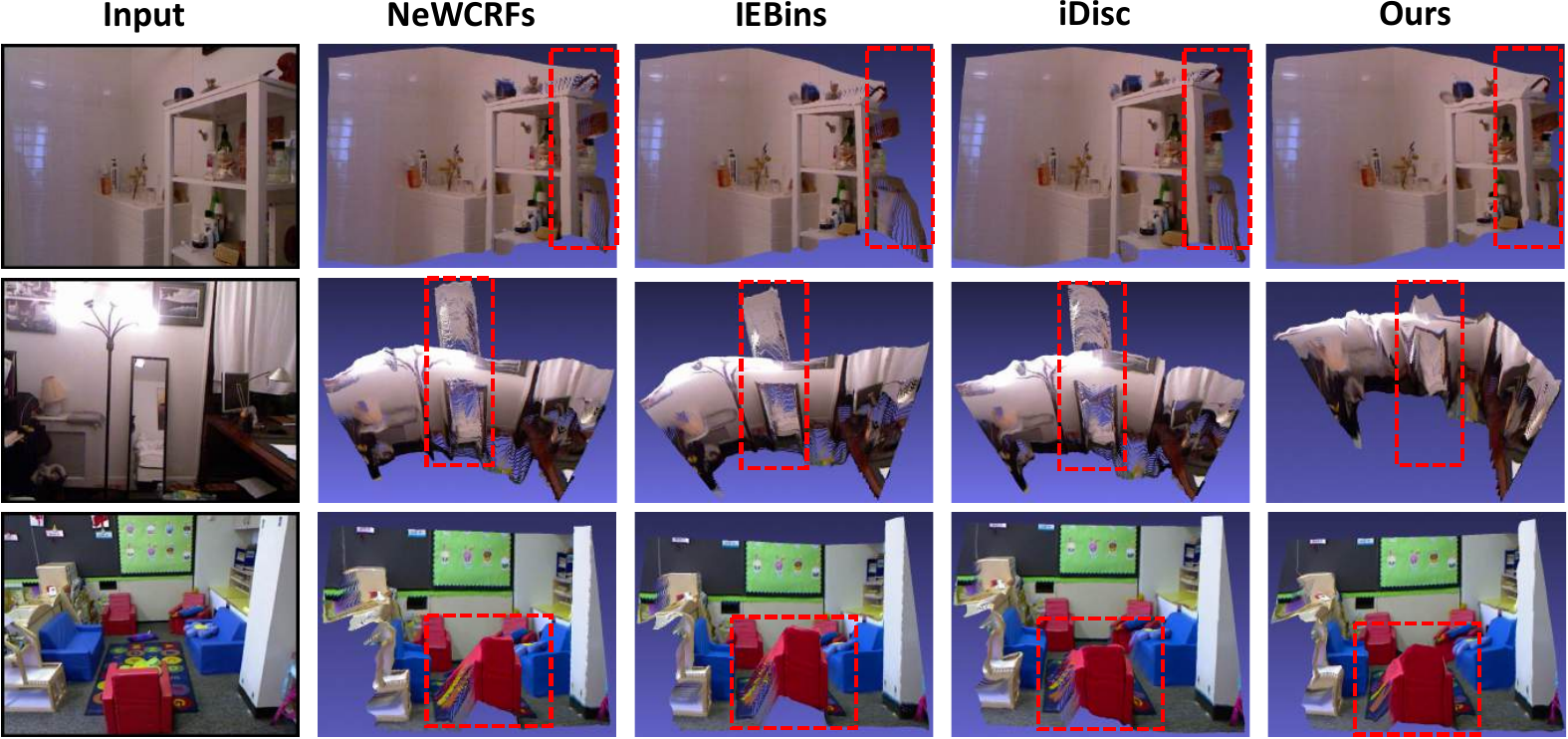}
\end{center}
\vspace{-3mm}
\caption{\textbf{Qualitative 3d shape comparison on the NYU dataset.} The red boxes indicate the regions
to focus on.}
\label{fig:nyu_shape_vis}
\vspace{-3mm}
\end{figure}

\begin{figure}[htb!]
\setlength{\abovecaptionskip}{-0.0cm}
\begin{center}
\includegraphics[width=0.98\linewidth]{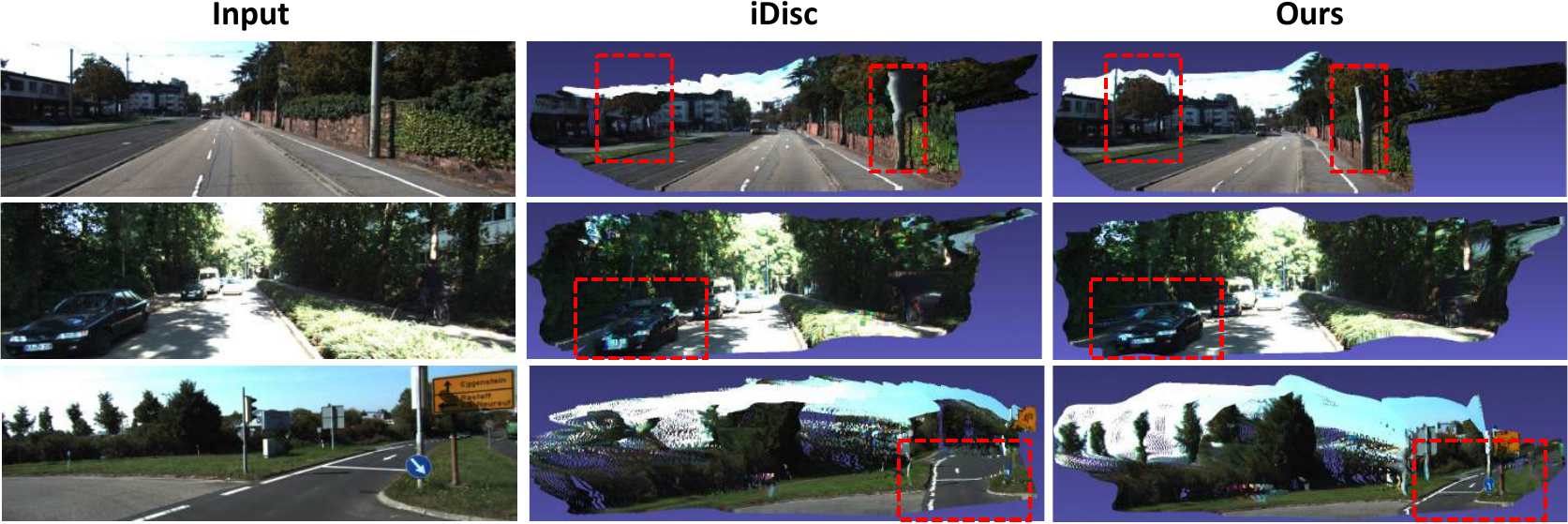}
\end{center}
\vspace{-3mm}
\caption{\textbf{Qualitative 3D shape comparison on the KITTI dataset.} The red boxes show the regions to focus on.}
\label{fig:kitti_shape_vis}
\vspace{-3mm}
\end{figure}

\subsection{3D Shape Recovery}
Tab.~\ref{tab:nyu_3d_metric} shows the performance comparison results of 3D shape recovery quality on NYU with other single-task MDE methods. We report 3D metrics including $F1$ score under various thresholds and Chamfer Distances on point clouds. Our method surpasses previous methods and achieves better results on all metrics. Fig.~\ref{fig:nyu_shape_vis} shows the qualitative point cloud comparison on NYU, where competing methods use additionally provided camera intrinsics for 3D shape recovery while we utilize our own estimated intrinsics. One can observe that our reconstructions have much less noise and outliers even with predicted intrinsics.
We present qualitative point cloud visualization comparison results on the Eigen-split of KITTI n Fig.~\ref{fig:kitti_shape_vis}. As can be seen, the proposed method shows less distortion than the compared approaches and recovers the structures of the 3D world reasonably.

\subsection{Ablation Study}

\textbf{Effectiveness of Key Components.}
Tab.~\ref{tab:nyu_ablation} shows the effectiveness of proposed components on NYU. We employ the naive combination of depth estimation and incident field estimation as the baseline (Row 2), which exhibits a performance decline in depth estimation. When equipped with the proposed canonical incident field $\mathbf{V}_{cano}$ (Row 3), one can observe a significant drop in FoV, which validates the effectiveness of our providing priors for incident field learning thus improving the performance of camera calibration. When adding the optimization in the 3D space (Row 4), \emph{i.e.}, $\mathcal{L}_{cd}$, the LSIV metric is further improved, which shows how point cloud optimization can help enhance 3D shape recovery. Overall, the ablation results show the effectiveness of the proposed strategies in 2D and 3D spaces.

\begin{table}[htb!]
\caption{\textbf{Ablation study of key components on NYU}.}
\label{tab:nyu_ablation}
\begin{center}
\small
\vspace{-5mm}
\setlength{\tabcolsep}{2mm}{
\begin{tabular}{c | c c  | c | c }	
\toprule
Method & RMSE $\downarrow$   & $\delta_1  $ $\uparrow$  & FoV $\downarrow$ & LSIV  $\downarrow$\\
\midrule
w/o Camera Head &  0.295 & 0.941  & - & - \\
\midrule
Baseline &  0.307 & 0.938  & 0.731 & 0.082 \\
Baseline+$\mathbf{V}_{cano}$ & 0.296  & 0.943  & 0.713 & 0.078 \\
Baseline+$\mathbf{V}_{cano}$+$\mathcal{L}_{cd}$  & \textbf{0.294}  & \textbf{0.944} & \textbf{0.709} & \textbf{0.074}\\
\bottomrule
\end{tabular}}
\end{center}
\vspace{-5mm}
\end{table}

\begin{table}[htb!]
\caption{\textbf{Comparisons of model parameters and inference time}.}
\label{tab:modelparam_time}
\vspace{-5mm}
\begin{center}
\small
\setlength{\tabcolsep}{3mm}{
\begin{tabular}{l | c |  c  c  c  }	
\toprule
Method &  $\mathcal{D}_{Chamfer}$ $\downarrow$ & Param(M) $\downarrow$ & Time(s) $\downarrow$\\
\midrule
NeWCRFs & 0.156 & 270 & \textbf{0.052} \\
IEBins & 0.128 & 273 & 0.085 \\
iDisc & 0.131 & \textbf{209} & 0.121  \\
\midrule
ours & \textbf{0.120} & 212 & 0.132  \\
\bottomrule
\end{tabular}}
\end{center}
\vspace{-5mm}
\end{table}

\begin{figure}[htb!]
\setlength{\abovecaptionskip}{-0.0cm}
\begin{center}
\includegraphics[width=1.0\linewidth]{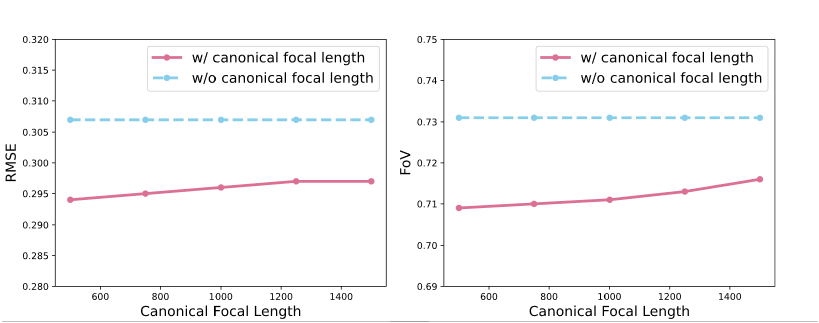}
\end{center}
\vspace{-3mm}
\caption{\textbf{Effect of canonical focal length on NYU dataset.}}
\label{fig:nyu_focal_length}
\vspace{-5mm}
\end{figure}

\textbf{Canonical Focal Length.}
We explore the impact of different canonical focal lengths that construct the canonical incidence field in our framework. Fig.~\ref{fig:nyu_focal_length} shows the results in terms of depth, focal length, and 3D shape on NYU. One can observe that the proposed canonical incident field is not sensitive to the canonical focal length. Although the performance declines slightly as the canonical focal length increases, all the metrics are still much better than not utilizing canonical focal length.

\subsection{Model Parameters and Inference Time}
Tab.~\ref{tab:modelparam_time} shows the comparison results of inference time and model parameters between the proposed method with other in-domain MDE methods using the Swin-Large backbone on the NYU dataset. It can be seen that the inference time of our method is slightly longer since it requires predicting the camera intrinsics while estimating depth. Nevertheless, our model parameters account for less than 80\% of IEBins and NeWCRFs. Meanwhile, the proposed method achieves the best 3D shape recovery quality even with the estimated camera intrinsics. Hence, our method provides a better balance between performance, number of parameters, and inference time.

\section{Conclusion and Future Work}
In this study, we reveal the reciprocal relations between depth and camera intrinsics and introduce a collaborative learning framework that jointly estimates depth maps and camera intrinsics from a single image. We propose a canonical incidence field mechanism and a shape similarity measurement loss thus achieving impressive performance on 3D shape recovery. Our CoL3D framework outperforms state-of-the-art in-domain MDE methods under the single-dataset setting while realizing outstanding camera calibration ability. In future work, we aim to expand our method to include training and evaluation on larger and more diverse datasets.




\bibliographystyle{IEEEtran}
\bibliography{IEEEfull}

\clearpage

\newpage
\section*{APPENDIX}
\section{Proof of Proposition}
\label{app:proof}
In this study, we explore the reciprocal relations between depth and camera intrinsics. Previous works~\cite{Facil_2019_CVPR,Yin_2023_ICCV,Guizilini_2023_ICCV} have shown that camera intrinsic enforces MDE models to implicitly understand camera models from the image appearance and then bridges the imaging size to the real-world size. This validates the guiding effect of camera intrinsics on the depth map. As a supplement from another perspective, we claim that depth serves as a 3D prior constraint on camera intrinsics estimation, which is revealed through the following proposition and proof. These two aspects demonstrate that depth and camera intrinsics are complementary and have a synergistic effect on each other.

\begin{proposition}
Given the depth map of an image, the 4 DoF camera intrinsics can be determined by 4 non-overlapping groups of pixels in the image with their Euclidean distances in the 3D space.
\end{proposition}

\begin{proof}
Assume that the depth map is $\mathbf{D}$, and the 4 groups of pixels and their Euclidean distances in the 3D space are formed as $\{(\mathbf{p}_{i1}, \mathbf{p}_{i2}), \mathbf{L}_{i}\}, i=1,2,3,4$. We denote the intrinsic matrix $\mathbf{K}$ of the camera model and its inverse matrix $\mathbf{K}^{-1}$ as:
\begin{equation}
\mathbf{K} = \left[
\begin{array}{ccc}
f_x & 0 & c_x \\
0 & f_y & c_y \\
0 & 0 & 1 \\
\end{array}\right],
\end{equation}
\begin{equation}
\mathbf{K}^{-1} = \left[
\begin{array}{ccc}
1/f_x & 0 & -c_x/f_x \\
0 & 1/f_y & -c_y/f_y \\
0 & 0 & 1 \\
\end{array}\right],
\end{equation}
where $f_x$ and $f_y$ are the pixel-represented focal length along the $x$ and $y$ axes, and $(c_x, c_y)$ is the principle center. Here, assuming that the camera is in ideal mode with no distortion.

Denote the homogeneous coordinate of a pixel $\mathbf{p^T} = [u\quad v\quad 1]$ in the 2D image space and its depth value $d = \mathbf{D}(\mathbf{p})$, the corresponding 3D point $\mathbf{P^T} = [X\quad Y\quad Z]$ is defined as:
\begin{equation}
\mathbf{P} = d \cdot \mathbf{K}^{-1} \mathbf{p} = d \cdot \left[\begin{array}{ccc}
(u-c_x)/f_x  \\
(v-c_y)/f_y  \\
1  \\
\end{array} \right].
\end{equation}
For a group of pixels $(\mathbf{p}_1, \mathbf{p}_2)$ and their Euclidean distance $L$ in the 3D space, we can get the following constraints:
\begin{equation}
\label{eq:origin}
\begin{aligned}
L^2 &= |\mathbf{P}_1\mathbf{P}_2|^2 \\
    &= \left[\frac{d_1(u_1-c_x)}{f_x} - \frac{d_2(u_2-c_x)}{f_x}\right]^2 \\
    & \quad + \left[\frac{d_1(v_1-c_y)}{f_y} - \frac{d_2(v_2-c_y)}{f_y}\right]^2 \\
    & \quad + (d_1 - d_2)^2.
\end{aligned}
\end{equation}
Arrange Eq.~\eqref{eq:origin}, we obtain:
\begin{equation}
\label{eq:arrange}
\begin{aligned}
& \frac{[d_1u_1-d_2u_2 + (d_2-d_1)c_x]^2}{f_x^2} \\
& + \frac{[d_1v_1-d_2v_2 + (d_2-d_1)c_y]^2}{f_y^2} \\
& + [(d_1 - d_2)^2 -L^2] = 0.
\end{aligned}
\end{equation}
Next, re-parametrize the unknowns in Eq.~\eqref{eq:arrange} to get:
\begin{equation}
\label{eq:target}
\frac{(a_1 + a_2c_x)^2}{f_x^2} + \frac{(a_3 + a_4c_y)^2}{f_y^2} + a_5 = 0,
\end{equation}
where $a_i (i=1,2,3,4,5)$ are constants. Expanding Eq.~\eqref{eq:target}, we obtain:
\begin{equation}
\frac{a_1^2}{f_x^2} + \frac{2a_1a_2c_x}{f_x^2} + \frac{a_2^2c_x^2}{f_x^2} + \frac{a_3^2}{f_y^2} + \frac{2a_3a_4c_y}{f_y^2} + \frac{a_4^2c_y^2}{f_y^2} + a_5 = 0.
\end{equation}
Let $t_x = \frac{c_x}{f_x}, t_y=\frac{c_y}{f_y}, r_x=\frac{1}{f_x}, r_y=\frac{1}{f_y}$, we have:
\begin{equation}
\label{eq:final}
a_1^2r_x^2 + 2a_1a_2t_xr_x + a_2^2t_x^2 + a_3^2r_y^2 + 2a_3a_4t_yr_y + a_4^2t_y^2 + a_5 = 0.
\end{equation}
By stacking Eq.~\eqref{eq:final} with $N = 4$ randomly sampled groups of pixels, we can acquire $N$ nonlinear equations where the intrinsic parameter to be solved is stored in the above 4 unknowns parameters $\{t_x, t_y, r_x, r_y\}$. This solves the other intrinsic parameters as:
\begin{equation}
f_x = \frac{1}{r_x}, f_y=\frac{1}{r_y}, c_x=\frac{t_x}{r_x}, c_y=\frac{t_y}{r_y}.
\end{equation}
If we choose $N=4$, we obtain a minimal solver where the solution is computed by performing the Levenberg-Marquard algorithm and the proof is over.
\end{proof}









\end{document}